\documentclass[nohyperref]{article}

\usepackage[ruled, algo2e]{algorithm2e}
\usepackage{amsthm}
\usepackage{amssymb}
\usepackage{amsmath, amsfonts}
\usepackage{bbm}
\usepackage{booktabs} 
\usepackage{graphicx}
\usepackage{hyperref}
\usepackage{mathtools}
\usepackage{microtype}
\usepackage{multirow}
\usepackage{subfigure}

\usepackage[capitalize,noabbrev]{cleveref}

\usepackage[accepted]{icml2022}

\theoremstyle{plain}
\newtheorem{theorem}{Theorem}[section]

\newtheorem{lemma}[theorem]{Lemma}
\newtheorem{corollary}[theorem]{Corollary}
\theoremstyle{definition}

\theoremstyle{remark}

\icmltitlerunning{Taming graph kernels with random features}

\begin{document}

\twocolumn[
\icmltitle{Taming graph kernels with random features}



\icmlsetsymbol{equal}{*}

\begin{icmlauthorlist}
\icmlauthor{Krzysztof Choromanski}{gdm,columbia}
\end{icmlauthorlist}

\icmlaffiliation{gdm}{Google Robotics}
\icmlaffiliation{columbia}{Columbia University}


\icmlcorrespondingauthor{Krzysztof Choromanski}{kchoro@google.com}

\icmlkeywords{Machine Learning, ICML}

\vskip 0.3in
]



\printAffiliationsAndNotice{}  

\begin{abstract}
We introduce in this paper the mechanism of \textit{graph random features} (GRFs). GRFs can be used to construct \textbf{unbiased} randomized estimators of several important kernels defined on graphs' nodes, in particular the \textit{regularized Laplacian kernel}. As regular RFs for non-graph kernels, they provide means to scale up kernel methods defined on graphs to larger networks. Importantly, they give substantial computational gains also for smaller graphs, while applied in downstream applications. Consequently, GRFs address the notoriously difficult problem of cubic (in the number of the nodes of the graph) time complexity of graph kernels algorithms. We provide a detailed theoretical analysis of GRFs and an extensive empirical evaluation: from speed tests, through Frobenius relative error analysis to kmeans graph-clustering with graph kernels. We show that the computation of GRFs admits an embarrassingly simple distributed algorithm that can be applied if the graph under consideration needs to be split across several machines. 
We also introduce a (still unbiased) \textit{quasi Monte Carlo} variant of GRFs, $q$-$GRFs$, relying on the so-called \textit{reinforced random walks} that might be used to optimize the variance of GRFs. 
As a byproduct, we obtain a novel approach to solve certain classes of linear equations with positive and symmetric matrices.
\end{abstract}
\vspace{-8mm}
\section{Introduction}
\label{sec:intro}

Consider a positive definite kernel (similarity function) $\mathrm{K}:\mathbb{R}^{d} \times \mathbb{R}^{d} \rightarrow \mathbb{R}$. Kernel methods \citep{kernel-6, kernel-5, kernel-4, kernel-3, kernel-2, kernel-1} provide powerful mechanisms for modeling non-linear relationships between data points. However, as relying on the so-called \textit{kernel matrices} $\mathbf{K}=[\mathrm{K}(\mathbf{x}_{i},\mathbf{x}_{j})]_{i,j}$ for datasets $\mathcal{X}=\{\mathbf{x}_{1},...,\mathbf{x}_{N}\}$ under consideration, they have time complexity at last quadratic in the size of a dataset. 

To address this limitation, a mechanism of \textit{random features} (RFs) \citep{RF-survey, rahimi2007random, rksink} to linearize kernel functions was proposed. Random features rely on randomized functions: $\phi:\mathbb{R}^{d} \rightarrow \mathbb{R}^{m}$ that map datapoints from the original to the new space, where the linear (dot-product) kernels corresponds to the original kernel. i.e: 
\begin{equation}
\label{eq:kernel-linearized}
\mathrm{K}(\mathbf{x},\mathbf{y}) = \mathbb{E}[\phi(\mathbf{x})^{\top}\phi(\mathbf{\mathbf{y}})]
\end{equation}
The theory of random features for kernels was born with the seminal paper \cite{jlt-base}, proposing a simple randomized linear transformation $\phi$ (with Gaussian random matrices) to approximate dot-product kernels via dimensionality reduction ($m \ll d$). That mapping $\phi$ is often referred to as the \textit{Johnson-Lindenstrauss Transform} (JLT).

Interestingly, it turned out that for several nonlinear kernels, JLT can be modified by adding a nonlinear functions acting element-wise on the linearly transformed dimensions of the input vectors (and potentially by replacing Gaussian sampling mechanism and altering the constant multiplicative renormalization term), to obtain Eq. \ref{eq:kernel-linearized}. For the radial basis function kernels (RBFs) \cite{rbfs} and softmax kernels, this nonlinear functions are trigonometric transformations \cite{rahimi2007random}. For the Gaussian (an instantiation of RBFs) and softmax kernels one can also apply the so-called \textit{positive random features} leveraging exponential nonlinearties \cite{choromanski2020rethinking}. For the angular kernel, the sign function is used \citep{goemans, choromanski2017unreasonable}. Thus linearization schemes for different kernels are obtained by applying different element-wise nonlinear transformations, but surprisingly, the overall structure of the mapping $\phi$ is the same across different kernels.

RFs provide \textbf{unbiased} low-rank decomposition of the kernel matrix $\mathbf{K}$, effectively enabling the (approximate) computations with $\mathbf{K}$ to be conducted in sub-quadratic time, by-passing explicit materialization of $\mathbf{K}$. Even more importantly, they help to create nonlinear variants of the simpler linear methods (e.g kernel-SVM \cite{svn-kernel} as opposed to regular SVM \cite{svn} or kernel regression \cite{kernel-ridge-regreession} as opposed to linear regression \cite{linear-regression}) via the $\phi$-transformation.

So far we have considered points embedded in the regular Euclidean $\mathbb{R}^{d}$-space, but what if the space itself is non-Euclidean ? Upon discretization, such a space can be approximated by undirected weighted graphs (often equipped with shortest-path-distance metric) and in that picture the points correspond to graph nodes. It is thus natural to consider in that context \textit{graph kernels} $\mathrm{K}:\mathrm{V} \times \mathrm{V} \rightarrow \mathbb{R}$ defined on the set of nodes $\mathrm{V}$ of the graph. And indeed, several such kernels were proposed: \textit{diffusion}, \textit{regularized Laplacian}, \textit{$p$-step random walk}, \textit{inverse cosine} and more \citep{smola-kondor, diff-kernels, chung}. We want to emphasize that the term \textit{graph kernels} is also often used for kernels taking as input graphs, thus computing similarities between two graphs rather than two nodes in a given graph \cite{gks}. This class of kernels is not a subject of this paper.

Unfortunately, for most graph kernels, the computation of the kernel matrix $\mathbf{K}$ is at least \textbf{cubic} in the number of nodes of the graph. This is a major obstacle on the quest to make those methods more practical and mainstream kernel techniques. A tempting idea is to try to apply random features also for graph kernels, yet so far no analogue of the mechanism presented above for regular kernels was proposed (see: Sec. \ref{sec:related} for the additional discussion on some efforts to apply RF-based techniques in the theory of graph kernels in the broad sense). This is in striking contrast to kernels operating in $\mathbb{R}^{d}$, where not only are RFs widely adopted, but (as already discussed) the core underlying mechanism is the same for a large plethora of different kernels, irrespectively of their taxonomy, e.g. shift-invariant (Gaussian, Laplace, Cauchy) or not (softmax, dot-product). One of the main challenges is that no longer is the value of the graph kernel solely the property of its input-nodes, but also the medium (the graph itself), where those nodes reside. In fact, it does not even make sense to talk about the properties of graphs' nodes by abstracting from the whole graph (in this paper we do not assume any internal graph-node structure).   

This work is one of the first steps to develop rigorous theory of random features for graph kernels. We introduce here the mechanism of \textit{graph random features} (GRFs). GRFs can be used to construct \textbf{unbiased} randomized estimators of several important kernels defined on graphs' nodes, in particular the \textit{regularized Laplacian kernel}. As regular RFs for non-graph kernels, they provide means to scale up kernel methods defined on graphs to larger networks. Importantly, they give substantial computational gains also for smaller graphs, while applied in downstream applications. Consequently, GRFs address the notoriously difficult problem of cubic (in the number of the nodes of the graph) time complexity of graph kernels algorithms. We provide a detailed theoretical analysis of GRFs and an extensive empirical evaluation: from speed tests, through Frobenius relative error analysis to kmeans graph-clustering with graph kernels. We show that the computation of GRFs admits an embarrassingly simple distributed algorithm that can be applied if the graph under consideration needs to be split across several machines. 
We also introduce a (still unbiased) \textit{quasi Monte Carlo} variant of GRFs, $q$-$GRFs$, relying on the so-called \textit{reinforced random walks} \cite{rrws} that might be used to optimize the variance of GRFs. 
As a byproduct, we also obtain a novel approach for solving linear equations with positive and symmetric matrices.

The q-GRF variants of our algorithm (that, as we show, still provide unbiased estimation) aim to translate the geometric techniques of structured random projections, that were recently successfully applied in the theory of RFs (see: Sec. \ref{sec:related} for detailed discussion) to non-Euclidean domains defined by graphs. To the best of our knowledge, it was never done before. We propose this idea in the paper, yet leave to future work comprehensive theoretical and empirical analysis.

\vspace{-3mm}
\section{Related work}
\label{sec:related}

Probabilistic techniques are used in different contexts in the theory of graph kernels. Sampling algorithms can be applied to make the computations of kernels taking graphs as inputs (rather than graphs' nodes) feasible \citep{deep-gks, efficient-graphlet, leskovec, Orbanz2017SubsamplingLG, bressan, sna-paper}. Examples include in particular \textit{graphlet} kernels, where counters of various small-size sub-graphs are replaced by their counterparts obtained via sub-graph sub-sampling (often via random walks) \citep{ribeiro, lingfeiwu, graphlet-faster}. They provide more tractable, yet biased estimation of the original objective. 

Other approaches define new classes of graph kernels starting from the RF-inspired representations. Examples include structure-aware random Fourier kernels from \cite{fourier-kernels} that apply the so-called \textit{$L$-hop sub-graphs} processed by graph neural networks (GNNs) to obtain their latent embeddings and define a kernel by taking their dot products. Even though here we proceed in a reverse order - starting from well-known graph kernels defined by deterministic closed-form expressions and producing their unbiased RF-based representations, there is an interesting link between that line of research and our work. GRFs can be cast as representations of similar core structure, but much more specific, GNN-free and consequently, providing simplicity, low computational time and strong theoretical guarantees. 

Quasi Monte Carlo (q-MC) methods, using correlated ensembles of samples are powerful MC techniques. New q-MC methods applying block-orthogonal ensembles of the random Gaussian vectors (the so-called \textit{orthogonal random features} or ORFs) were shown to improve various RF-based kernels' estimators \citep{yu2016orthogonal, dpp-orfs, grfs, choromanski2017unreasonable, choromanski2020rethinking, likhosherstov2022chefs, psrn}.

\section{Graph random features (GRFs)}
\label{sec:algorithms}
\subsection{Hidden Gram-structure of squared inverse matrices}
\label{sec:gram}

We start our analysis, that will eventually lead us to the definition of GRFs, by forgetting (only temporarily) about graph theory and focusing on positive symmetric matrices.
\subsubsection{Preliminaries}
Consider a positive symmetric matrix $\mathbf{U} \in \mathbb{R}^{N \times N}$ of spectral radius 
$\rho(\mathbf{U}) \overset{\mathrm{def}}{=} \max_{\lambda \in \Lambda(\mathbf{U})} |\lambda|$ satisfying: 
$\rho(\mathbf{U}) < 1$, where $\Lambda(\mathbf{U})$ stands for the set of eigenvalues of $\mathbf{U}$ ($\Lambda(\mathbf{U}) \subseteq \mathbb{R}$ since $\mathbf{U}$ is symmetric). Note that under this condition, the following series converges:
\begin{equation}
\mathbf{I}_{N}+\mathbf{U} + \mathbf{U}^{2} + ...
\end{equation}
and is equal to $(\mathbf{I}_{N}-\mathbf{U})^{-1}$. Furthermore, we have:
\begin{lemma}
\label{first-lemma}
If $\mathbf{U} \in \mathbb{R}^{N \times N}$ is a a positive symmetric matrix with $\rho(\mathbf{U}) < 1$ then the following holds:
\begin{equation}
\mathbf{I}_{N} + 2\mathbf{U} + 3\mathbf{U}^{2} + ... = (\mathbf{I_{N}-\mathbf{U}})^{-2}
\end{equation}
\begin{proof}
We have: $\mathbf{I}_{N} + 2\mathbf{U}+3\mathbf{U}^{2} + ... = (\mathbf{I}_{N}+\mathbf{U}+\mathbf{U}^{2}+...)+(\mathbf{U}+2\mathbf{U}^{2}+3\mathbf{U}^{3}+...)$. Thus we have:
\begin{align}
\begin{split}
\mathbf{I}_{N} + 2\mathbf{U}+3\mathbf{U}^{2} + ... =
(\mathbf{I}_{N}+\mathbf{U}+\mathbf{U}^{2}+...) + \\ \mathbf{U}(\mathbf{I}_{N} + 2\mathbf{U}+3\mathbf{U}^{2} + ...)
\end{split}
\end{align}
Therefore, from what we have said above, we obtain:
\begin{align}
\begin{split}
(\mathbf{I}_{N}-\mathbf{U})(\mathbf{I}_{N} + 2\mathbf{U}+3\mathbf{U}^{2} + ...) = (\mathbf{I}_{N}-\mathbf{U})^{-1}
\end{split}
\end{align}
and, consequently: $\mathbf{I}_{N} + 2\mathbf{U}+3\mathbf{U}^{2} + ...=(\mathbf{I}_{N}-\mathbf{U})^{-2}$.
\end{proof}
\end{lemma}
We will derive an algorithm for unbiasedly estimating $(\mathbf{I}_{N}-\mathbf{U})^{-2}$ with the Gram-matrix $\mathbf{M} = [\phi(i)^{\top}\phi(j)]_{i,j=1,...,N}$ for a randomized mapping: $\phi:\mathbb{N} \rightarrow \mathbb{R}^{N}$. We will refer to $\phi(k)$ for a given $k \in \{1,...,N\}$ as a \textit{signature vector} for $k$ (see also Fig. \ref{fig:phi} for the illustration of signature vectors).

\subsubsection{Computing signature vectors}
\label{sec:sign-vectors}

We are ready to bring back the graphs. We interpret $\mathbf{U}=[u_{i,j}]_{i,j=1,...,N}$ as a weighted adjacency matrix of the weighted graph $\mathrm{G}_{\mathbf{U}}$ with vertex set $\mathrm{V}=\mathrm{V}(\mathrm{G}_{\mathbf{U}})$ of $N$ vertices, where the weight between node $i$ and $j$ is given as $u_{i,j}$. We then calculate signature vectors via random walks on $\mathrm{G}_{\mathbf{U}}$. The algorithm for computing $\phi(i)$ for a specific node $i$ is given in Algorithm 1 box. Calculations of $\phi(i)$ for different $i$ are done independently (and therefore can be easily parallelized). The algorithm works as follows.

\begin{algorithm}[h]
\SetAlgoLined
\SetKwInOut{Input}{Input}
\Input{graph $\mathrm{G}_{\mathbf{U}}$, termination probability $p_{\mathrm{term}}$, sampling strategy $\mathrm{sample}$, \# of random walks $m$, node $i$.}
\KwResult{signature vector $\phi(i)$.}
\textbf{Main: \\}
1. Initialize: $\phi(i)[j]=0$ for $j=1,2,...,N$.\\
2. Initialize: $\mathcal{H} \leftarrow \emptyset$.\\
3. \textbf{for } $t=1,...,m$ \textbf{do}\\
\phantom{xxx} \textbf{(a)} initialize: $\mathrm{load}=1$ \\
\phantom{xxx} \textbf{(b)} initialize: $\mathrm{current}\_\mathrm{vertex} \leftarrow i$ \\
\phantom{xxx} \textbf{(c)} update: $\phi(i)[i] \leftarrow \phi(i)[i]+ 1$ \\
\phantom{xxx} \textbf{(d)} \textbf{while} (\textrm{not } \textrm{terminated}) \textbf{do} \\
\phantom{xxxxxxx} 1. assign: $v=\mathrm{current}\_\mathrm{vertex}$ \\
\phantom{xxxxxxx} 2. $w, p = \mathrm{sample}
(v,\mathrm{G}_{\mathbf{U}},\mathcal{H})$ \\
\phantom{xxxxxxx} 3. update: $\mathrm{current}\_\mathrm{vertex}  \leftarrow w$ \\
\phantom{xxxxxxx} 4. update: $\mathrm{load}  \leftarrow \mathrm{load} \cdot u_{v,w} \textrm{ }/\textrm{ } p(1-p_{\mathrm{term}})$ \\
\phantom{xxxxxxx} 5. update: $\phi(i)[w] \leftarrow \phi(i)[w] + \mathrm{load}$ \\
\phantom{xxxxxxx} 6. update: $\mathcal{H} \leftarrow \mathcal{H}.\mathrm{add}((v,w))$ \\
4. renormalize: $\phi(i) \leftarrow \phi(i) \textrm{ } / \textrm{ } m$\\
\caption{Computing a signature vector for a given $i$.}
\label{alg:1}
\end{algorithm}
The signature vector is zeroed out in initialization. In the initialization, we also zero-out data structure \textit{history} $\mathcal{H}$ which maintains some particular function of the list of visited edges (more details later). A number $m$ of random walks, all initiated at point $i$ is conducted. In principle, they can also
be parallelized, but here, for the clarity of the analysis, we assume that they are executed 
in a given order. 

\begin{figure}[!htb]
\vspace{-3mm}
  \includegraphics[width=0.95\linewidth]{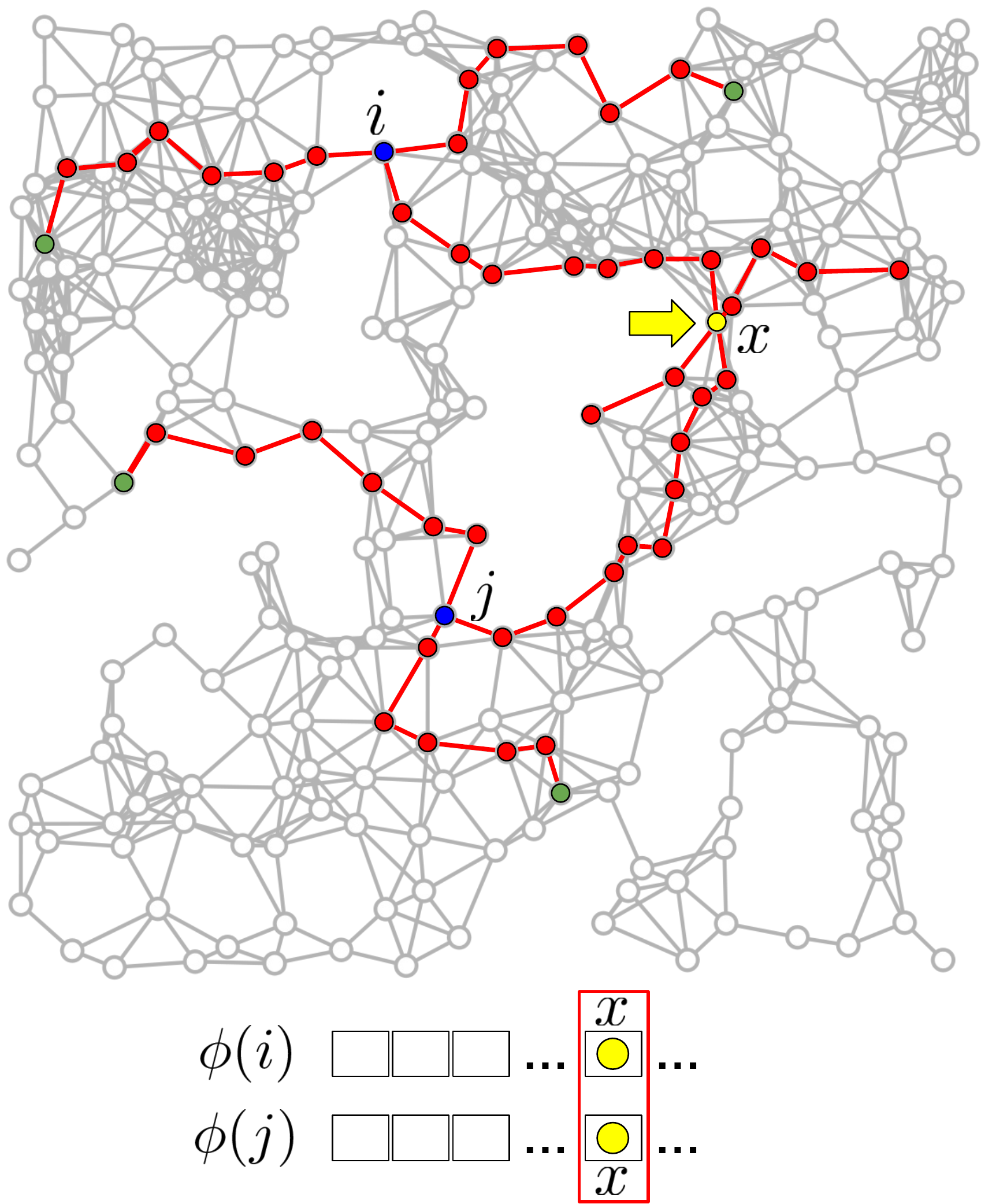}
\vspace{-4mm}
\caption{\small{The pictorial representation of signature vectors. The contibutions to the dot-product between two signature vectors come from the vertices that were visited by both vectors.}}
\label{fig:phi}
\end{figure}

Each random walk carries a $\mathrm{load}$ that is being updated every time the termination state has not been reached. Before each transition, the termination is reached independently with a given probability $p_{\mathrm{term}}$.
The update-rule of the $\mathrm{load}$ is given in point 4 in the Algorithm 1 box. Strategy $\mathrm{sample}$ takes as input: (1) the current vertex $v$, (2) entire graph, (3) the function of the history of visits and outputs: (1) a neighbor $w$ of $v$ according to a particular randomized sampling strategy as well as: (2) its corresponding probability $p$. Every time a new vertex is reached, the value of its corresponding signature vector dimension is increased by the most updated value of the $\mathrm{load}$, as given in point 5 in the Algorithm 1 box.
When all the walks terminate, the signature vector is renormalized by the multiplicative term $\frac{1}{m}$ and returned.
\vspace{-3mm}
\paragraph{Sampler \& the history of visits:} The simplest $\mathrm{sample}$ strategy is to choose a neighbor $w$ of a current vertex $v$ uniformly at random from the set of all neighbors $\mathrm{N}_{v}$ of $v$ in $\mathrm{G}_{\mathbf{U}}$. In this setting: $p=\frac{1}{\mathrm{deg}(v)}$, where $\mathrm{deg}(v)$ stands for the degree of $v$ defined as the size $|\mathrm{N}_{v}|$ of its neighborhood. In that setting, $\mathcal{H}$ is not needed and thus we have: $\mathcal{H} = \emptyset$. We show in Sec. \ref{sec:experiments} that this strategy works very well in practice. Another option is to take: $p=p(v,w)=\frac{u_{v,w}}{\sum_{z \in \mathrm{N}_{v}}u_{v,z}}$, i.e. make probabilities proportional to weights. In that setting, history is also not needed. In Sec. \ref{sec:thm}, we analyze in more detail the q-GRF setting, where nontrivial $\mathcal{H}$ is useful. This history can be used to diversify the set of random walks, by de-prioritizing edges that were already selected in previous random walks. Thus it might have the same positive effect on the estimation accuracy as ORFs from \cite{yu2016orthogonal} that diversify the set of random projection vectors. However detailed analysis of such strategies is beyond the scope of this paper. Most importantly, as we show next, regardless of the particular instantiations of $\mathrm{sample}$, signature vectors computed in Algorithm 1 indeed provide an unbiased estimation of the matrix $(\mathbf{I}_{N}-\mathbf{U})^{-2}$ (proof in Sec. \ref{sec:thm}):

\begin{theorem}
\label{thm:core}
Denote by $\mathbf{B} \in \mathbb{R}^{N \times N}$ a matrix with rows given by $\phi(i)^{\top}$ for $i=1,...,N$, computed according to Algorithm 1 and by $\mathbf{B}^{\prime}$ its independent copy (obtained in practice by constructing another indeepentent set of random walks). Assume that method $\mathrm{sample}$ is not degenerate, i.e. it assigns a positive probability $p(v,w)$ for every neighbor $w$ of $v$. Then the following is true:
\begin{equation}
(\mathbf{I}_{N}-\mathbf{U})^{-2} = \mathbb{E}[\mathbf{B}(\mathbf{B}^{\prime})^{\top}].
\vspace{-2.5mm}
\end{equation}
\end{theorem}
\paragraph{Time complexity \& signature vector maps:} To compute $\phi(i)$ for every $i$, we need to run in every node $m$ random walks, each of the expected length $l=\frac{1}{p_{\mathrm{term}}}$, thus the expected time complexity of computing $\mathbf{B},\mathbf{B}^{\prime}$ is $O(\frac{Nmt_{\mathrm{sam}}}{p_{\mathrm{term}}}+T_{\mathrm{rep}})$, where: $t_{\mathrm{sam}}$ stands for time complexity of the mechanism $\mathrm{sample}$ and $T_{\mathrm{rep}}$ for the space complexity of storing the representation of $\mathrm{G}_{\mathbf{U}}$. For instance, if $\mathbf{U}$ is sparse, graph adjacency list representation is a right choice. In principle, even more compact representations are possible if graph-edges do not need to be stored explicitly. 
For the most straightforward uniform sampling strategy, we have: $t_{\mathrm{sam}}=O(1)$ and for many more sophisticated ones $t_{\mathrm{sam}}$ is proportional to the average degree of a vertex.
We assume that in all considered representations there is an indexing mechanism for the neighbors of any given vertex and a vertex of a particular index can be retrieved in $O(1)$ time. Note that for $p_{\mathrm{term}} \gg \frac{m t_{\mathrm{sam}}}{N}$, $\mathbf{B},\mathbf{B}^{\prime}$ should be kept in the implicit manner (this is a pretty conservative lower bound; in practice different walks will share many vertices), since most of the entries of each of its rows are equal to zero in this case. Thus for every $i$, we can store a dictionary with keys corresponding to visited vertices $w$ and values to the values of $\phi_{i}[w]$. We call such a dictionary a \textit{signature vector map} (or $\mathrm{svmap}$). In practice, one can choose large $p_{\mathrm{term}}$ without accuracy drop (see: Sec. \ref{sec:experiments}).   

\subsubsection{Trimming RFs: anchor points and JLT}
\label{sec:anchor_jlt}

Algorithm 1 provides substantial computational gains in calculating $(\mathbf{I}_{N}-\mathbf{U})^{-2}$, replacing $O(N^{3})$ time complexity with expected $O(\frac{Nm}{p_{\mathrm{term}}}+T_{\mathrm{rep}})$. Furthermore, signature vectors $\phi(i)$ in practice can be often stored efficiently (see: our discussion on $\mathrm{svmaps}$ above). Interestingly, there are also easy ways to directly control their "effective dimensionality" via one of two simple to implement tricks, provided below.
\vspace{-3mm}
\paragraph{Anchor points:} This technique samples randomly $K < N$ points from the set of all $N$ vertices of $\mathrm{G}_{\mathbf{U}}$. We call this set: \textit{anchor points}. Signature vectors are updated only in anchor points (but we still upload $\mathrm{load}$ for each transition) and thus effectively each of them has dimensionality $K$. In principle, different sampling strategies are possible and analyzing all of them in outside the scope of this work. The simplest approach is to sample uniformly at random a $K$-element subset of $\mathrm{V}$ with no repetitions. In this setting, if we replace term $\frac{1}{m}$ in point 4. of the Algorithm 1 box with $\frac{N}{Km}$ term, the estimator remains unbiased (see: Sec. \ref{sec:thm}).
\vspace{-3mm}
\paragraph{JLTed signature vectors:} The more algebraic approach to trimming the dimensionality of signature vectors is to apply Johnson-Lindenstrauss Transform (see: Sec. \ref{sec:intro}). We simply replace $\phi(i)$ with: $\phi_{\mathrm{JLT}}(i) = \frac{1}{\sqrt{K}}\mathbf{G}\phi(i)$ after all the computations. Here $\mathbf{G} \in \mathbb{R}^{K \times N}$ is the Gaussian matrix with entries taken independently at random from $\mathcal{N}(0,1)$ and $K$ is the number of RFs used in JLT. Matrix $\mathbf{G}$ is sampled independently from all the walks and is the same for all the nodes. The unbiasedness of the dot-product kernel estimation with JLTs \cite{johnson1984extensions} combined with Theorem \ref{thm:core}, provides the unbiasedness of the overall mechanism.

\subsection{From squared inverse matrices to graph kernels}

\subsubsection{Preliminaries}

We are ready to show how the results obtained in Sec. \ref{sec:gram} can be applied for graph kernels. 
For a given undirected weighted loopless graph $\mathrm{G}$ with a weighted symmetric adjacency matrix $\mathbf{W} \in \mathbb{R}^{N \times N}$, we consider the family of \textit{$d$-regularized Laplacian kernels} defined as follows for
$\mathrm{deg}_{\mathbf{W}}(i) \overset{\mathrm{def}}{=} \sum_{j \in \mathrm{N}_{i}}\mathbf{W}(i,j)$: 
\begin{equation}
\label{eq:d-kernels}
[\mathrm{K}_{\mathrm{lap}}^{d}(v,w)]_{i,j=1,...,N} = (\mathbf{I}_{N}+\sigma^{2}\tilde{\mathbf{L}}_{\mathrm{G}})^{-d},
\end{equation}
for a symetrically normalized Laplacian $\tilde{\mathbf{L}}_{\mathrm{G}} \in \mathbb{R}^{N \times N}$:
\begin{equation}
\tilde{\mathbf{L}}_{\mathrm{G}}(v,w) =
\begin{cases}
    1       & \quad \text{if } v=w\\
    -\frac{\mathbf{W}(v,w)}{\sqrt{\mathrm{deg}_{\mathbf{W}}(v)\mathrm{deg}_{\mathbf{W}}(w)}}  & \quad \text{if } (v,w) \textrm{ is an edge}
\end{cases}
\end{equation}

For $d=1$, we get popular regularized Laplacian kernel. As we show in Sec. \ref{sec:app_lap}, as long as the $1$-regularized Laplacian kernel is well defined (i.e. matrix $\mathbf{I}_{N}+\sigma^{2}\tilde{\mathbf{L}}_{\mathrm{G}}$ is invertible and the inverse has positive entries), the $d$-regularized Laplacian kernel is a valid positive definite kernel.

\subsubsection{Signature vectors are RFs for the $2$-regularized Laplacian kernel}

We will start with $d=2$ since it turns out that the mechanism of signature vectors can be straightforwardly applied there to obtain corresponding GRFs. Note that the following is true for $\mathbf{U} \in \mathbb{R}^{N \times N}$ defined as:
$\mathbf{U}=[u_{i,j}]_{i,j=1,...,N}$, where $u_{i,j} = \frac{\sigma^{2}}{\sigma^{2}+1}\frac{\mathbf{W}(i,j)}{\sqrt{\mathrm{deg}_{\mathbf{W}}(i)\mathrm{deg}_{\mathbf{W}}(j)}}$:
\begin{equation}
\label{eq:case_two}
\mathbf{I}_{N}+\sigma^{2}\tilde{\mathbf{L}}_{\mathrm{G}} = (\sigma^{2}+1)(\mathbf{I}_{N} - \mathbf{U})
\end{equation}

Equation $\ref{eq:case_two}$, together with Theorem \ref{thm:core}, immediately provide the GRF-mechanism for the $2$-regularized Laplacian kernel:

\begin{corollary}
\label{cor:case_two}
Consider a kernel matrix of the $2$-regularied Laplacian kernel for a given graph $\mathrm{G}$ with $N$ nodes $\{1,2,...,N\}$.
Then the following holds:
\begin{equation}
(\mathbf{I}_{N}+\sigma^{2}\tilde{\mathbf{L}}_{\mathrm{G}})^{-2} = \mathbb{E}[\mathbf{C}(\mathbf{C}^{\prime})^{T}],
\end{equation}
where the rows of $\mathbf{C} \in \mathbb{R}^{N \times K}$ are of the form: $\frac{1}{\sigma^{2}+1}\phi(i)^{\top}$ for $i=1,...,N$ and signature vectors $\phi(i)$ computed via Algorithm 1 for the graph $\mathrm{G}_{\mathbf{U}}$ and a matrix $\mathbf{U}$ defined above. Furthermore, $\mathbf{C}^{\prime}$ is an independent copy of $\mathbf{C}$.
\end{corollary}

\subsubsection{GRFs \& regularized Laplacian kernel}

Let us now consider the case $d=1$. From Corollary \ref{cor:case_two}, we immediately get the following:
\begin{corollary}
\label{cor:case_one}
Consider a kernel matrix of the $1$-regularied Laplacian kernel for a given graph $\mathrm{G}$ with $N$ nodes $\{1,2,...,N\}$.
Then the following holds:
\begin{equation}
(\mathbf{I}_{N}+\sigma^{2}\tilde{\mathbf{L}}_{\mathrm{G}})^{-1} = \mathbb{E}[\mathbf{C}\mathbf{D}^{T}],
\end{equation}
where matrix $\mathbf{D} \in \mathbb{R}^{N \times K}$ is defined as: $\mathbf{D} = (\mathbf{I}_{N}+\sigma^{2}\tilde{\mathbf{L}}_{\mathrm{G}})\mathbf{C}^{\prime}$ and $\mathbf{C},\mathbf{C}^{\prime} \in \mathbb{R}^{N \times K}$ are as in Corollary \ref{cor:case_two}.
\end{corollary}

Corollary \ref{cor:case_one} provides a GRF mechanism for the regularized Laplacian kernel, since it implies that we can write 
$\mathrm{K}^{1}_{\mathrm{lap}}$ as:
\begin{equation}
\mathrm{K}^{1}_{\mathrm{lap}}(i,j) = \mathbb{E}[\phi(i)^{\top}\psi(j)],
\end{equation}
where $\phi(i)$ and $\psi(i)$ for $i=1,...,N$ are the rows of $\mathbf{C}$ and $\mathbf{D}$ respectively. Notice that this is the so-called \textit{asymmetric RF setting} \cite{hrfs} (where we use a different transformation for $i$ and $j$). The approximate kernel matrix is not necessarily symmetric, but its expectation is always symmetric and the estimation is unbiased. However it is also easy to obtain here the GRF mechanism providing symmetric approximate kernel matrix and maintaining unbiasedness. It suffices to define:
\begin{align}
\begin{split}
\phi^{\prime}(i) = \frac{1}{\sqrt{2}}[\phi(i)^{\top},\psi(i)^{\top}]^{\top},\\
\psi^{\prime}(j) = \frac{1}{\sqrt{2}}[\psi(j)^{\top},\phi(j)^{\top}]^{\top},\\
\end{split}
\end{align}
and the corresponding approximate kernel matrix is of the form: $\frac{1}{2}(\mathbf{C}\mathbf{D}^{\top}+\mathbf{D}\mathbf{C}^{\top})$.

\subsubsection{Going beyond $d=2$}
Let us see now what happens when we take $d>2$.
We will define GRFs recursively, having the variants for $d=1$ and $d=2$ already established in previous sections (they will serve as a base for our induction analysis). Take some $d$ and assume that the following GRF mechanism defined by the unbiased low-rank decomposition of the $d$-regularized Laplacian kernel matrix exists for some $\mathbf{X},\mathbf{Y} \in \mathbb{R}^{N \times K}$:
\begin{equation}
(\mathbf{I}_{N}+\sigma^{2}\tilde{\mathbf{L}}_{\mathrm{G}})^{-d} = \mathbb{E}[\mathbf{X}\mathbf{Y}^{\top}]
\end{equation}
Then, using Corollary \ref{cor:case_two}, we can rewrite:
\begin{equation}
(\mathbf{I}_{N}+\sigma^{2}\tilde{\mathbf{L}}_{\mathrm{G}})^{-d-2}= \mathbb{E}[\mathbf{X}\mathbf{Y}^{\top}\mathbf{C}(\mathbf{C}^{\prime})^{\top}],
\end{equation}
if we assume that matrices $\mathbf{C},\mathbf{C}^{\prime} \in \mathbb{R}^{N \times K}$ are constructed independently from $\mathbf{X}$ and $\mathbf{Y}$.
We can then compute in time $O(NK^{2})$ matrix $\mathbf{S}=\mathbf{Y}^{\top}\mathbf{C}$ and then conduct its SVD-decomposition:
$\mathbf{S} = \mathbf{Z}\Sigma\mathbf{Z}^{\top}$ in time $O(K^{3})$ for $\mathbf{Z},\Sigma \in \mathbb{R}^{K \times K}$ and where $\Sigma$ is diagonal. Then we can rewrite:
\begin{equation}
(\mathbf{I}_{N}+\sigma^{2}\tilde{\mathbf{L}}_{\mathrm{G}})^{-d-2}= 
\mathbb{E}[(\mathbf{XZ}\Sigma^{\frac{1}{2}})(\mathbf{C}^{\prime}\mathbf{Z}\Sigma^{\frac{1}{2}})^{\top}]
\end{equation}
for $\mathbf{XZ}\Sigma^{\frac{1}{2}}, \mathbf{C}^{\prime}Z\Sigma^{\frac{1}{2}} \in \mathbb{R}^{N \times K}$. The random feature vectors can be then given as the rows of $\mathbf{XZ}\Sigma^{\frac{1}{2}}$ and $\mathbf{C}^{\prime}Z\Sigma^{\frac{1}{2}}$.

We see that starting with $d=1,2$ we can then follow the above procedure to produce GRF-based low-rank decomposition of the kernel matrix for the $d$-regularized Laplacian kernel for any $d \in \mathbb{N}_{+}$.

\subsubsection{GRFs for fast graph kernel methods}

We will assume here that the graphs $\mathrm{G}$ under consideration have $e$ edges 
for $e=o(N^{2})$ which is a reasonable assumption in machine learning and thus graph adjacency list representation is a default choice. Using the analysis from Sec. \ref{sec:sign-vectors}, we conclude that GRFs can be computed in time $O(\frac{Nmt_{\mathrm{sam}}}{p_{\mathrm{term}}}+Ne)$ for $d=2$. If uniform sampling is the strategy used by $\mathrm{sample}$ then time complexity can be simplified to $O(\frac{Nm}{p_{\mathrm{term}}}+Ne)$. If JLT-based dimensionality is applied, an additional cost $O(N^{2}K)$ is incurred. This can be further reduced to $O(N^{2}\log(K))$ if unbiased structured JLT variants are applied \cite{choromanski2017unreasonable}. For 
$d=1$, an additional time $O(Ne)$ (for computing matrix $\mathbf{D}$, see: Corollary \ref{cor:case_one}) is incurred. Similar analysis can be applied to larger $d$. It is easy to see that for $d>2$ GRFs can be computed in time cubic in $K$ and linear in $d$. Thus for $d>2$ they are useful if $K \ll N$. We call this stage \textbf{graph pre-processing}. If the brute-force approach is applied, pre-processing involves explicit computation of the graph kernel matrix and for the graph kernels considered in this paper (as well as many others) takes time $O(N^{3})$. In some applications, pre-processing is a one-time procedure (per training), but in many others needs to be applied several times in training (e.g. in graph Transformers, where graph kernels can be used to topologically modulate regular attention mechanism) \cite{block-toeplitz}.

When the graph pre-processing is completed, graph kernels usually interact with the downstream algorithm via their matrix-vector multiplication interface, e.g. the algorithm computes $\mathbf{Kx}$ for a given graph kernel matrix $\mathbf{K}=[\mathrm{K}(i,j)]_{i,j=1,...,N} \in \mathbf{R}^{N \times N}$ and a series of vectors $\mathbf{x} \in \mathbb{R}^{N}$ (see: KMeans clustering with graph kernels: \cite{kernel-kmeans}). We will refer to this phase as \textbf{inference}. Brute-force inference can be trivially computed in time $O(N^{2})$ per matrix-vector multiplication.
Using GRFs strategy, an \textbf{unbiased} estimation of the matrix-vector product can be computed via a series of matrix-vector multiplications obtained from (a chain of) decompositions  constructed to define GRFs. For $d \neq 2$ those decompositions involve at least three matrices and thus the GRFs do not even need to be explicitly constructed.

For instance, for $d=2$, inference can be conducted in time $O(NK)$ and this is a pretty conservative bound. Even if one applies $K = \Omega(N)$, most of the entries of random feature vectors can be still equal to zero (if $p_{\mathrm{term}}$ is not too small) and in such a setting $\mathrm{svmaps}$ from Sec. \ref{sec:sign-vectors} can be applied to provide sub-quadratic in $N$ time complexity. For $d=1$, an additional cost $O(Ne)$ of multiplications with matrices $\mathbf{I}_{N}+\sigma^{2}\tilde{\mathbf{L}}_{\mathrm{G}}$ needs to be incurred. For $d>2$, inference can be directly conducted in time $O(NK)$.
\vspace{-2mm}
\subsubsection{Systems of linear equations \& GRFs}

It is easy to see that there is nothing particularly special about the Laplacian  matrix 
$\mathbf{I}_{N}+\sigma^{2}\tilde{\mathbf{L}}_{\mathrm{G}}$ from Corollary \ref{cor:case_one}
and that for any invertible matrix of the form: $\lambda(\mathbf{I}_{N}-\mathbf{U})$ and symmetric $\mathbf{U} \in \mathbb{R}^{N \times N}$ of $o(N^{2})$ nonzero entries, one can find a decomposition:
\begin{equation}
\vspace{-3mm}
\frac{1}{\lambda}(\mathbf{I}_{N}-\mathbf{U})^{-1} = \mathbb{E}[\mathbf{C}\mathbf{D}^{\top}]
\end{equation}
in sub-cubic time, using our methods. That however provides immediately a sub-cubic randomized algorithm for unbiasedly solving linear systems of the form: $(\mathbf{I}_{N}-\mathbf{U})\mathbf{x}=\mathbf{b}$ for any given $\mathbf{b} \in \mathbb{R}^{N}$
as: $\mathbf{x}=\lambda (\mathbf{C}(\mathbf{D}^{\top}\mathbf{b}))$, where brackets indicate the order of computations.

\subsubsection{Final remarks}

The construction of GRFs presented in this section can be thought of as a low-rank decomposition of the kernel matrix $\mathbf{K}$. Then one can argue that instead of using presented algorithm, a standard low-rank decomposition method can be applied (such as spectral analysis). This however defeats the main goal of this paper - sub-cubic time complexity (the matrix to be decomposed would need to be explicitly constructed) and unbiased estimation (such an approximation would be biased). 
The latter remains a problem for the alternative approaches even if more refined low-rank decomposition techniques not enforcing materialization of the kernel matrix $\mathbf{K}$ are applied (such methods would however rely on efficient algorithms for computing $\mathbf{Kx}$ which are non-trivial: for $d=1$ would need to solve systems of linear equations with symmetric matrices as sub-routines).

\section{Theoretical results}
\label{sec:thm}

\subsection{GRFs provide unbiased graph kernel estimation}

We start by proving Theorem\ref{thm:core}. Concentration results for GRFs in the base setting, where the sampler chooses a neighbor uniformly at random, are given in Sec. \ref{sec:appendix_concentration}. We leave as an exciting open question the analysis of the concentration results beyond second moment methods.
\begin{proof}
For any two nodes $a,b \in \mathrm{V}(\mathrm{G}_{\mathbf{U}})$, denote by $\Omega(a,b)$ the set of walks from $a$ to $b$ in $\mathrm{G}_{\mathbf{U}}$. 
Furthermore, symbol $\subseteq_{\mathrm{pre}}$ indicates than one walk is a prefix of the other.
Then:
\begin{align}
\begin{split}
\label{eq:unbiased}
\phi(i)^{\top}\phi(j) = \frac{1}{m^{2}} \sum_{x \in \mathrm{V}(\mathrm{G}_{\mathbf{U}})} \sum_{k=1}^{m}\sum_{l=1}^{m}
\sum_{\omega_{1} \in \Omega(i,x)} \\ \sum_{\omega_{2} \in \Omega(j,x)}  w(\omega_{1}) \cdot w(\omega_{2}) 
\prod_{s=1}^{l_{1}} \frac{(p^{i}_{k,s})^{-1}}{1-p_{\mathrm{term}}} \prod_{t=1}^{l_{2}} \frac{(p^{j}_{l,t})^{-1}}{1-p_{\mathrm{term}}} \\
\mathbbm{1}[\omega_{1} \subseteq_{\mathrm{pre}} \bar{\Omega}(k,i)]\mathbbm{1}[\omega_{2} \subseteq_{\mathrm{pre}} \bar{\Omega}(l,j)],
\end{split}
\end{align}
where $\bar{\Omega}(k,i)$ and $\bar{\Omega}(l,j)$ stand for the $k^{th}$ and $l^{th}$ sampled random walk from $i$ and $j$ respectively, function $w$ outputs the product of the edge-weights of its given input walk and 
$p^{i}_{k,s}$ and $p^{j}_{l,t}$ are the probabilities returned by $\mathrm{sample}$ for the $s^{th}/t^{th}$ vertex of the $k^{th}/l^{th}$ sampled random walk starting at $i/j$. Furthermore, $l_{1}$ and $l_{2}$ refer to the number of edges of $\omega_{1}$ and $\omega_{2}$ respectively. Note that those probabilities are in general random variables (if they are chosen as nontrivial functions of the history of sampled walks). 

Therefore, for $\mathrm{len(\omega)}$ denoting the number of edges of $\omega$:
\begin{align}
\begin{split}
\mathbb{E}[\phi(i)^{\top}\phi(j)] = \sum_{x \in \mathrm{V}(\mathrm{G}_{\mathbf{U}})}
\sum_{\omega_{1} \in \Omega(i,x)} \\ \sum_{\omega_{2} \in \Omega(j,x)} w(\omega_{1})w(\omega_{2})
=\sum_{\omega \in \Omega(i,j)}(\mathrm{len(\omega)}+1)w(\omega)
\end{split}
\end{align}
The proof is completed by an observation that:
\begin{equation}
(\mathbf{I}_{N}+2\mathbf{U}+3\mathbf{U}^{2} + ...)(i,j) 
=\sum_{\omega \in \Omega(i,j)}(\mathrm{len(\omega)}+1)w(\omega)
\end{equation}
This in turn follows from the fact that for any $k \geq 0$: $\mathbf{U}^{k}(i,j)$ can be interpreted
as the sum of $w(\omega)$ over all walks $\omega$ from $i$ to $j$ of $k$ edges.
\end{proof}
\vspace{-3mm}
The extension to the anchor point setting is obtained by following the same proof, but with x iterating only over sampled anchor points and results in adding a multiplicative constant, as explained in Sec. \ref{sec:anchor_jlt}. 
\vspace{-3mm}
\subsection{q-GRFs and correlated random walks}

As already mentioned, the idea of q-GRFs is to correlate different random walks initiated in a given node $i$, to "more efficiently" explore the entire graph (in particular, by avoiding reconstructing similar walks) and can be thought of as an analogue of the ORF or low discrepancy sequences techniques applied for regular RFs \cite{qmc-rf}. A natural candidate for the q-GRF variant is the one, where method $\mathrm{sample}$ implements the so-called \textit{reinforced random walk} strategy \cite{kozma}. The probability of choosing a neighbor $w$ of $v$ (see: Sec. \ref{sec:sign-vectors}) can be defined as:
\begin{equation}
p(v,w) = \frac{f(N(v,w))}{\sum_{z \in \mathrm{N}_{v}} f(N(v,z))},
\end{equation}
where $\mathrm{N}(x,y)$ stands for the number of times that an edge $\{x,y\}$ has been already used (across all the walks, or previous walks) and $f:\mathbb{N} \rightarrow \mathbb{R}$ is a fixed function. Note that, since we want to deprioritize edges that were already frequently visited, $f$ should be a decreasing function.

Implementing such a strategy can be done straightforwardly with an additional multiplicative factor $d_{\mathrm{ave}}$ in time-complexity, where $d_{\mathrm{ave}}$ stands for the average degree of a graph node. We leave detailed theoretical and empirical analysis of this class of methods to future work.

\vspace{-3mm}
\section{Experiments}
\label{sec:experiments}

We empirically verify the quality of GRFs via various experiments, including downstream applications of graph kernels.

\subsection{Speed tests}
\label{sec:speed}

Here we compared the total number of FLOPS used by the GRF-variant of of the regularized Laplacian kernel in pre-processing followed by a single inference call involving computing the action of the kernel matrix on a given vector (this is for instance the way in which kernels are applied in the kernelized KMeans algorithm, see: Sec. \ref{sec:kmeans}) with the correesponding time for various linear system solvers as well as a brute-force algorithm. Note that for $d$-regularized Laplacian kernels with $d=1$, inference straightforwardly reduces to solving linear systems of equations. We used graphs of different size. The results are presented in Table 1. GRFs provide substantial reduction of FLOPS, even for smaller graphs. We took $p_{\mathrm{term}}=0.1$ since it worked well in several other tests (see: Sec. \ref{sec:frobenius}, Sec. \ref{sec:kmeans}).

\begin{table}[!htb]
    \begin{center}
    \begin{footnotesize}
        \scalebox{0.90}{\begin{tabular}{@{}lccccc@{}}
        \toprule
         GRF & BF & Gauss-Seidel & Conjugate Gradient & Jacobi \\
        \midrule
         960K & 512.64M & 6400K & 512M & 6400K \\
        \midrule
        1.4M & 1001M & 10M & 1000M & 10M \\
        \midrule
         10.2M & 27B & 90M & 27B & 90M \\
        \midrule
        \bottomrule
        \end{tabular}}
    \end{footnotesize}
    \vspace{-3mm}
    \caption{The comparison of the number of FLOPS for the setting from Sec. \ref{sec:speed} for different linear systems solvers and GRFs. Different rows correspond to: $N=800$, $N=1000$ and $N=3000$.}
    \end{center}
\label{tab:app}
\end{table}

\begin{figure*}[!htb]
  \includegraphics[width=\linewidth]{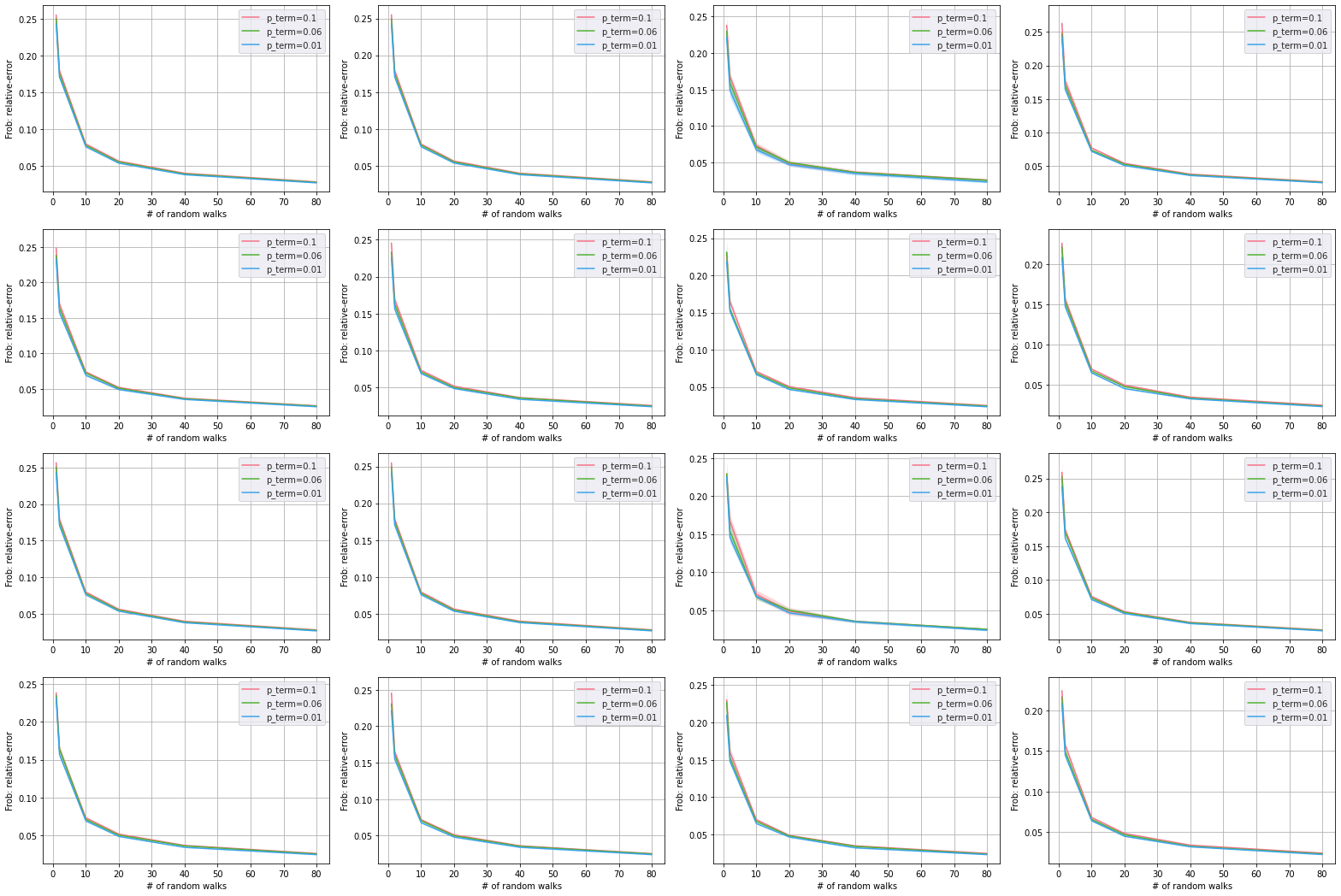}
\vspace{-9mm}
\caption{\small{Relative Frobenius norm error for the setting described in Sec. \ref{sec:frobenius}. First two rows correspond to $d=1$ and last two to $d=2$. We considered the following graphs (from upper-left to lower-right): Erdos-Renyi graph with edge prob. 0.4 (ER-0.4, $N=1000$),
ER-0.1 ($N=1000$), \textrm{dolphins} ($N=62$), \textrm{eurosis} ($N=1272$), \textrm{Networking} ($N=1249$), \textrm{Databases} ($N=1046$), \textrm{Encryption}-\textrm{and}-\textrm{Compression} 
 ($N=864$), \textrm{Hardware}-\textrm{and}-\textrm{Architecture} ($N=763$).}}
\label{fig:frob}
\end{figure*}

\subsection{Relative Frobenius norm error}
\label{sec:frobenius}

We took several undirected graphs of different sizes and computed for them groundtruth kernel matrices $\mathbf{K}^{d}=[\mathrm{K}^{d}_{\mathrm{lap}}(i,j)]_{i,j=1,...,N}$ using $d$-regularized Laplacian kernels with $d=1,2$. We then computed their counterparts $\widehat{\mathbf{K}}^{d}$ obtained via GRFs and the corresponding relative Frobenius norm error defined as:
\begin{equation}
\epsilon = \frac{\|\mathbf{K}^{d}-\widehat{\mathbf{K}}^{d}\|_{\mathrm{F}}}{\|\mathbf{K}^{d}\|_{\mathrm{F}}}
\end{equation}
We considered different termination probabilities: $p_{\mathrm{term}} \in \{0.1, 0.06, 0.01\}$ leading to average random walk lengths: $10, \frac{50}{3}, 100$ respectively as well as different number $m$ of random walks for the construction of the GRF-vector: $m=1,2,10,20,40,80$. We fixed: $\sigma^{2}=0.2$. The reported empirical relative Frobenium norm errors were obtained by averaging over $s=10$ independent experiments. We also reported their corresponding standard deviations.

The results are presented in Fig. \ref{fig:frob}. We see that different $p_{\mathrm{term}}$ result in very similar error-curves and thus in practice it suffices to take $p_{\mathrm{term}}=0.1$ (average walk length $l=10$), regardless of the density of the graph, even for graphs with $N>1000$ nodes. Interestingly, error-curves are also very  
similar across all the graphs. Furthermore, small number of random walks $m=80$ suffices to obtain $\epsilon < 2\%$. The standard deviations are reported on the plots, but since they are very small, we also include them in Sec. \ref{seec:app_exp}.
\vspace{-3mm}
\subsection{KMeans algorithm with graph kernels}
\label{sec:kmeans}

We applied GRFs in kernelized KMeans algorithm \cite{kernel-kmeans} to cluster graph nodes using graph kernels.
As in the previous section, we applied $d$-regularized Laplacian graph kernels with $d=1,2$. Furthermore, we have measured the impact of the additional dimensionality reduction techniques (anchor points and JLT-based reduction) on the clustering quality. For each variant, we reported the so-called \textit{clustering error} defined as $\epsilon = \frac{P_{\mathrm{error}}}{{N \choose 2}}$, where $\mathrm{P_{\mathrm{error}}}$ stands for the number of pairs of nodes $\{i,j\}$ of the graph that were classified differently by the groundtruth algorithm and its GRF-variant (e.g. they belonged to the same cluster in the groundtruth variant and two different clusters in the GRF one or vice versa). In all the experiments we used $\sigma^{2}=0.2$,  $p_{\mathrm{term}}^{-1}m \leq 400$ and $K \leq 0.6N$. We chose the no of clusters $\mathrm{nb}\_\mathrm{clusters}=3$. The results are in Table 2. GRFs, even with more accurate of the: anchor points and JLTs methods provide accurate approximation, for some graphs (e.g. \textrm{citeseer} with $N=2120$) with error $<1\%$.
\vspace{-2.5mm}
\begin{table}[!htb]
    \begin{center}
    \begin{footnotesize}
        \scalebox{0.80}{\begin{tabular}{@{}lcccc@{}}
        \toprule
        \textbf{Kernel, method} & \textrm{citeseer} & \textrm{Databases} & \textrm{polbooks} & \textrm{karate} \\
        \midrule
        {d=1, \textrm{reg}} & 0.020 & 0.170 & 0.28 & 0.11   \\
        {d=1, \textrm{min}(\textrm{jlt}, \textrm{anc})} & 0.010 & 0.097 & 0.323 & 0.314   \\
        \midrule
        {d=2, \textrm{reg}} & 0.008 & 0.140 & 0.12 & 0.032   \\ 
        {d=2, \textrm{min}(\textrm{jlt}, \textrm{anc})} & 0.010 & 0.098 & 0.264 & 0.198   \\ 
        \midrule
        \bottomrule
        \end{tabular}}
    \end{footnotesize}
    \vspace{-3mm}
    \caption{Clustering errors for the experiments from Sec. \ref{sec:kmeans} (for the $d$-regularized Laplacian kernel). Different tested methods: (1) regular GRFs ($\textrm{reg}$, no compression), (2) GRFs with the more accurate among: JLT ($\textrm{jlt}$), and anchor points ($\textrm{anc}$) techniques.}
    \end{center}
\label{tab:app}
\end{table}


\section{Conclusion}
\label{sec:conclusion}
We have proposed in this paper a new paradigm of graph random features (GRFs) for the unbiased and computationally efficient estimation of various kernels defined on the nodes of the graph. GRFs rely on the families of random walks initiated in different nodes, depositing loads in the visited vertices of the graph. The computation of GRFs can be also easily paralellized.  We have provided detailed theoretical analysis of our method and exhaustive experiments, involving in particular kmeans clustering on graphs.

\bibliography{references}
\bibliographystyle{icml2023}

\newpage
\appendix
\onecolumn

\section{Appendix}
\label{sec:appendix}

\subsection{Additional experimental details}
\label{seec:app_exp}
In this section, we report standard deviations for the experiments conducted in Sec. \ref{sec:frobenius}. The results are presented in Table 3.
\begin{table}
    \begin{center}
    \begin{footnotesize}
        \scalebox{0.7}{\begin{tabular}{@{}lcccccc@{}}
        \toprule
        \textbf{Kernel, Graph, $p_{\mathrm{term}}$} & m=1 & m=2 & m=10 & m=20 & m=40 & m=80 \\
        \midrule

        {d=1, ER-0.4-1000, 0.1} & 0.0013625896465110972 & 0.0007259063790268863 & 0.0001903398359981377 & 0.00011763026820191872 & 8.79471573831236e-05 & 4.29321194947306e-05 \\
        {d=1, ER-0.4-1000, 0.06} & 0.0007220475266445387 & 0.0005435175726472963 & 0.00013151409601958063 & 9.575658877244471e-05 & 4.3906087463097626e-05 & 4.0755077316485244e-05 \\
        {d=1, ER-0.4-1000, 0.01} & 0.0003671696761976026 & 0.00022105611802712477 & 9.098043091730836e-05 & 7.626805728093205e-05 & 3.874109025559475e-05 & 3.7022093881032226e-05 \\
        \midrule
        {d=1, ER-0.1-1000, 0.1} & 0.015608995370960629 & 0.00616767931977833 & 0.004336727765953796 & 0.003934111721006881 & 0.0027504647840658337 & 0.0013502125948369828\\
        {d=1, ER-0.1-1000, 0.06} & 0.017347549335088377 & 0.011169983712293842 & 0.004103948449818662 & 0.002123051971417801 & 0.0015286651824786544 & 0.001809200930986754 \\
        {d=1, ER-0.1-1000, 0.01} & 0.010199780651542109 & 0.008926804414936252 & 0.003593843664943383 & 0.002372436626864906 & 0.002528684013474355 & 0.0013236120502376739 \\
        \midrule
        {d=1, dolphins, 0.1} & 0.015608995370960629 & 0.00616767931977833 & 0.004336727765953796 & 0.003934111721006881 & 0.0027504647840658337 & 0.0013502125948369828 \\
        {d=1, dolphins, 0.06} & 0.017347549335088377 & 0.011169983712293842 & 0.004103948449818662 & 0.002123051971417801 & 0.0015286651824786544 & 0.001809200930986754 \\
        {d=1, dolphins, 0.01} & 0.010199780651542109 & 0.008926804414936252 & 0.003593843664943383 & 0.002372436626864906 & 0.002528684013474355 & 0.0013236120502376739 \\
        \midrule
        {d=1, eurosis, 0.1} & 0.014053698765343884 & 0.0060276707312881305 & 0.00092538589272756 & 0.0006342279681653932 & 0.0004384555835094276 & 0.00018536188888052218 \\
        {d=1, eurosis, 0.06} & 0.013536611535644354 & 0.004751929441074538 & 0.001020005565545354 & 0.000889042372679839 & 0.0005086950138860196 & 0.0002556955468331719 \\
        {d=1, eurosis, 0.01} & 0.012713452113533774 & 0.0037293050211431988 & 0.0007734429750918471 & 0.0006679713650390155 & 0.0004663147033666988 & 0.00029012600990115697 \\
        \midrule
        {d=1, Networking, 0.1} & 0.0146915283223426 & 0.00354911912464601 & 0.0008569845357120783 & 0.0007800162218792217 & 0.000486001643398867 & 0.0003402086731748513 \\
        {d=1, Networking, 0.06} & 0.009070791531361084 & 0.004330451321925086 & 0.0011314015990526772 & 0.0009955841652321848 & 0.0005356434310080785 & 0.0002993746309815702 \\
        {d=1, Networking, 0.01} & 0.007969253518861505 & 0.0035153340017524855 & 0.0015484554089057877 & 0.0009127004247725001 & 0.00024227482240693019 & 0.0002548717948222871 \\
        \midrule
        {d=1, Databases, 0.1} & 0.006982200176332185 & 0.004920247514564321 & 0.0010238037860984882 & 0.0006558270450890158 & 0.00046304862763166114 & 0.0003245028197695966 \\
        {d=1, Databases, 0.06} & 0.0035234877250858113 & 0.004526013840391158 & 0.000963926605473951 & 0.0007692921395363883 & 0.0005198262247493105 & 0.0004310588594538151 \\
        {d=1, Databases, 0.01} & 0.006381620629869565 & 0.004943075493204197 & 0.0008015939178439096 & 0.000780276410600827 & 0.0003613319043567048 & 0.0003127005621608197 \\
        \midrule
        {d=1, Enc\&Comp, 0.1} & 0.0078678393915265 & 0.0048278192814453745 & 0.0010825306683229289 & 0.0008869476479790293 & 0.0007716890723512931 & 0.0003704820179027577 \\
        {d=1, Enc\&Comp, 0.06} & 0.0190931216162695 & 0.0020893222009769137 & 0.0012842314605827185 & 0.0004723607601931406 & 0.0006878263190053518 & 0.0003750554030785581 \\
        {d=1, Enc\&Comp, 0.01} & 0.01836840019303472 & 0.00504474860734935 & 0.0013024704440231251 & 0.0009294029364896054 & 0.000660959802814907 & 0.00020474232102263236 \\
        \midrule
        {d=1, Hard\&Arch, 0.1} & 0.006704189653139892 & 0.004611366876905434 & 0.0009512468516057303 & 0.0012081728737839681 & 0.0007626404640344821 & 0.00040508532045909195 \\
        {d=1, Hard\&Arch, 0.06} & 0.005209803104840105 & 0.005063954232425586 & 0.001065299727136788 & 0.0012631414043132634 & 0.0006979985117773265 & 0.000421532809850973 \\
        {d=1, Hard\&Arch, 0.01} & 0.006795583742146745 & 0.00274347868585054 & 0.0007186205865051595 & 0.0007365735158683657 & 0.00045503009550719795 & 0.00040966993435182316 \\
        \midrule
        \midrule
        \midrule
        {d=2, ER-0.4-1000, 0.1} & 0.0010997014167145687 & 0.0005611717968776918 & 0.00027274809547358435 & 0.00011320568751584864 & 6.757544191642404e-05 & 6.039762637275414e-05 \\
        {d=2, ER-0.4-1000, 0.06} & 0.0008458403328815557 & 0.0004256206092542753 & 0.0001255914552193544 & 8.877177278468484e-05 & 4.854379424147713e-05 & 4.386961653719971e-05 \\
        {d=2, ER-0.4-1000, 0.01} & 0.0005840737941493646 & 0.0003048387437939227 & 8.612938808993734e-05 & 8.861321498224955e-05 & 3.118698564760755e-05 & 2.6440001535795706e-05 \\
        \midrule
        {d=2, ER-0.1-1000, 0.1} & 0.0013198005978007825 & 0.0005236255151350093 & 0.00022767131458694158 & 0.00013017955748542738 & 9.70665903350761e-05 & 4.911946309935634e-05 \\
        {d=2, ER-0.1-1000, 0.06} & 0.000923414702719891 & 0.0005865169497721389 & 0.00025410090630458796 & 0.00019059306793781872 & 0.00011444739288479127 & 7.032284297507503e-05 \\
        {d=2, ER-0.1-1000, 0.01} & 0.00033652461295461897 & 0.00046667795838897455 & 0.00027540444251076277 & 0.00010490895428890273 & 0.00013378373593132866 & 7.724500570322393e-05 \\
        \midrule
        {d=2, dolphins, 0.1} & 0.015313307957016969 & 0.008072537726967633 & 0.00571208308059399 & 0.0051191001841705405 & 0.0017181464186888107 & 0.001090082709032633 \\
        {d=2, dolphins, 0.06} & 0.015764248711711566 & 0.009755110708905598 & 0.00389170462370401 & 0.0025002668985388667 & 0.0016741448929614132 & 0.001371113891426613 \\
        {d=2, dolphins, 0.01} & 0.0152287323244012 & 0.006479716073801618 & 0.0024690593909838824 & 0.0016779353904653569 & 0.0018280376993101016 & 0.0008140444788820552 \\
        \midrule
        {d=2, eurosis, 0.1} & 0.01536989939709328 & 0.002985106085538542 & 0.000876946183305746 & 0.0007680934798181345 & 0.0005249300647113408 & 0.0002331259532404271 \\
        {d=2, eurosis, 0.06} & 0.012033428935866071 & 0.0041183672021622205 & 0.0008723728098050511 & 0.0007700295259631496 & 0.0005934744828608486 & 0.00017965740908675555 \\
        {d=2, eurosis, 0.01} & 0.007908289233031857 & 0.003124420288948241 & 0.0012421356753991252 & 0.000522089061052929 & 0.00036706376215862906 & 0.00019123151674677616 \\
        \midrule
        {d=2, Networking, 0.1} & 0.009327696246334569 & 0.003894812593269389 & 0.0012115481023431956 & 0.000687373457322389 & 0.0003709238725270653 & 0.0002917580174452871 \\
        {d=2, Networking, 0.06} & 0.008527029896162816 & 0.004866701669257291 & 0.0010784522973902714 & 0.0007199911928308271 & 0.000574302270555022 & 0.00040290056035626387 \\
        {d=2, Networking, 0.01} & 0.009587174636811118 & 0.003824906545448091 & 0.0014371944936820223 & 0.0003065956016769412 & 0.0003731510983670436 & 0.0003803384907960821 \\
        \midrule
        {d=2, Databases, 0.1} & 0.006982200176332185 & 0.004920247514564321 & 0.0010238037860984882 & 0.0006558270450890158 & 0.00046304862763166114 & 0.0003245028197695966 \\
        {d=2, Databases, 0.06} & 0.0035234877250858113 & 0.004526013840391158 & 0.000963926605473951 & 0.0007692921395363883 & 0.0005198262247493105 & 0.0004310588594538151 \\
        {d=2, Databases, 0.01} & 0.006381620629869565 & 0.004943075493204197 & 0.0008015939178439096 & 0.000780276410600827 & 0.0003613319043567048 & 0.0003127005621608197 \\
        \midrule
        {d=2, Enc\&Comp, 0.1} & 0.018895082497842376 & 0.006724706042286198 & 0.0014845981818248232 & 0.0010458737350497948 & 0.0008081389763835953 & 0.00035644825231180745 \\
        {d=2, Enc\&Comp, 0.06} & 0.01378196975159284 & 0.004100807796697153 & 0.0011827406795993489 & 0.0008530016288033847 & 0.0007814930040936276 & 0.00024164229327689848 \\
        {d=2, Enc\&Comp, 0.01} & 0.0038593031848660315 & 0.002887642103947921 & 0.0009055525946103308 & 0.0006739056569936388 & 0.0005584100220945653 & 0.00042133127128199215 \\
        \midrule
        {d=2, Hard\&Arch, 0.1} & 0.00591662612899018 & 0.003095039050038908 & 0.00126431072738253 & 0.000531917678481419 & 0.00047539738617717955 & 0.0005823370820758574 \\
        {d=2, Hard\&Arch, 0.06} & 0.0062802558107981615 & 0.0028895869476586425 & 0.0017287614112698189 & 0.0012669302897452082 & 0.0005934700975376661 & 0.0003548026338104648 \\
        {d=2, Hard\&Arch, 0.01} & 0.007928049506258186 & 0.0029574141141546616 & 0.0014192560156932896 & 0.0008739945261346017 & 0.0006409254479486006 & 0.0005212426748758263 \\
        \bottomrule
        \end{tabular}}
    \end{footnotesize}
    \vspace{-3mm}
    \caption{Standard deviations for the experiments from Sec. \ref{sec:frobenius} (for the $d$-regularized Laplacian kernel).}
    \end{center}
\label{tab:app}
\end{table}

\subsection{$D$-regularized Laplacian kernels}
\label{sec:app_lap}

We show here that as long as the regularized Laplacian kernel is well-defined, the $d$-regularized variants for $d>1$ are valid positive definite kernels. The following is true:

\begin{theorem}
If a matrix $\mathbf{I}_{N}+\sigma^{2}\tilde{\mathbf{L}}_{\mathrm{G}}$ is invertible and the inverse has positive entries, then the $d$-regularized Laplacian kernel is a valid positive definite kernel.
\end{theorem}

\begin{proof}
Denote $\mathbf{A}=\sigma^{2}\tilde{\mathbf{L}}_{\mathrm{G}}$.
Note first that, by Perron-Frobenius Theorem, we have: $\rho(\mathbf{A}) = \lambda_{\mathrm{max}}(\mathbf{A})$, where $ \lambda_{\mathrm{max}}(\mathbf{A})$ stands for the largest eigenvalue of $\mathbf{A}$ (note that all eigenvalues of $\mathbf{A}$ are real since $\mathbf{A}$ is symmetric). Then for $n$ sufficiently large we have: $\|\mathbf{A}\|^{n} \leq (1-\xi)^{k}$ for some $\xi \in (0,1)$. Thus the Neumann series: $\sum_{n=0}^{\infty} \mathbf{A}^{n}$ converges to: $(\mathbf{I}_{N}-\mathbf{A})^{-1}$. Thus we can rewrite: $(\mathbf{I}_{N}-\mathbf{A})^{-d} = (\sum_{n=0}^{\infty} \mathbf{A}^{n})^{d}$. Therefore we have: $(\mathbf{I}_{N}-\mathbf{A})^{-d} = \sum_{n}c_{d,n}\mathbf{A}^{n}$ for some sequence of coefficients: $(c_{d,0},c_{d,1},...)$. We conclude that $(\mathbf{I}_{N}-\mathbf{A})^{-d}$ is symmetric for any $d=1,2,...$ (since $\mathbf{A}$ is). We can rewrite: $(\mathbf{I}_{N}-\mathbf{A})^{-(d+1)}=\mathbf{X}\mathbf{Y}$, where: $\mathbf{X}=(\mathbf{I}_{N}-\mathbf{A})^{-d}$ and $\mathbf{Y}=(\mathbf{I}_{N}-\mathbf{A})^{-1}$.
We will proceed by induction on $d$. We can then conclude that matrix $(\mathbf{I}_{N}-\mathbf{A})^{-(d+1)}=\mathbf{X}\mathbf{Y}$ is a product of two positive definite matrices. Furthermore, $(\mathbf{I}_{N}-\mathbf{A})^{-(d+1)}$ is symmetric, as we have already observed. But then it is a positive definite matrix.
\end{proof}

\subsection{Concentration results for GRFs with the base sampler}
\label{sec:appendix_concentration}

We inherit the notation from the main body of the paper, in particular from the proof of Theorem \ref{thm:core}.
We also denote: $A(\omega)=\prod_{s=1}^{l_{1}} \frac{\mathrm{deg(v_{s})}}{1-p_{\mathrm{term}}}$, where: $\omega=(v_{1},v_{2},...,v_{l_{1+1}})$. Furthermore, for two walks: $\omega_{1},\omega_{2}$,we use the following notation:
$\omega_{1} \subset_{\mathrm{pre}} \omega_{2}$ if $\omega_{1}$ is a \textbf{strict} prefix of $\omega_{2}$ and $\omega_{1} \subseteq_{\mathrm{pre}} \omega_{2}$ if $\omega_{1}$ is a prefix (not necessarily strict) of $\omega_{2}$.
We prove the following concentration result.

\begin{theorem}
Consider an unbiased estimation of the matrix $\mathbf{K}=(\mathbf{I}_{N}-\mathbf{U})^{-2}$ via $\mathbf{M}=\mathbf{B}(\mathbf{B}^{\prime})^{\top}$ for matrices $\mathbf{B},\mathbf{B}^{\prime} \in \mathbb{R}^{N \times N}$, as in Theorem \ref{thm:core}. Then the following is true for any $i,j \in \{1,...,|\mathrm{V}(\mathrm{G}_{\mathbf{U}})|\}$, $i \neq j$:
\begin{equation}
\mathrm{Var}(\mathbf{M}(i,j)) = \frac{1}{m^{2}}(\Lambda-\mathbf{K}^{2}(i,j))
\end{equation}
for $\Lambda$ defined as follows:
\begin{equation}
\Lambda = \sum_{x_{1} \in \mathrm{V}}\sum_{x_{2} \in \mathrm{V}}
\sum_{\omega_{1} \in \Omega(i,x_{1})}\sum_{\omega_{2} \in \Omega(j,x_{1})}\sum_{\omega_{3} \in \Omega(i,x_{2})}
\sum_{\omega_{4} \in \Omega(j,x_{2})} w(\omega_{1})w(\omega_{2})w(\omega_{3})w(\omega_{4}) \Gamma(\omega_{1},\omega_{2},\omega_{3},\omega_{4})
\end{equation}
and where:
\begin{equation}
\Gamma(\omega_{1},\omega_{2},\omega_{3},\omega_{4}) =
\begin{cases}
    \frac{1}{A(\omega_{2})A(\omega_{4})}       & \quad \text{if } \omega_{1} \subset_{\mathrm{pre}} \omega_{2} \textrm{ and } \omega_{3} \subset_{\mathrm{pre}} \omega_{4}\\
    \frac{1}{A(\omega_{2})A(\omega_{3})}       & \quad \text{if } \omega_{1} \subset_{\mathrm{pre}} \omega_{2} \textrm{ and } \omega_{4} \subseteq_{\mathrm{pre}} \omega_{3}\\
    \frac{1}{A(\omega_{1})A(\omega_{4})}       & \quad \text{if } \omega_{2} \subseteq_{\mathrm{pre}} \omega_{1} \textrm{ and } \omega_{3} \subset_{\mathrm{pre}} \omega_{4}\\
    \frac{1}{A(\omega_{1})A(\omega_{3})}       & \quad \text{if } \omega_{2} \subseteq_{\mathrm{pre}} \omega_{1} \textrm{ and } \omega_{4} \subseteq_{\mathrm{pre}} \omega_{3}\\
\end{cases}
\end{equation}
\end{theorem}
\begin{proof}
We have the following:
\begin{equation}
\mathrm{Var}(\mathbf{M}(i,j)) = \frac{1}{m^{2}} \sum_{k=1}^{m}\sum_{l=1}^{m} X_{k,l},
\end{equation}
where: 
\begin{equation}
X_{k,l}=\sum_{x \in \mathrm{V}}\sum_{\omega_{1} \in \Omega(i,x)}\sum_{\omega_{2} \in \Omega(j,x)}w(\omega_{1})w(\omega_{2})
\mathbbm{1}[\omega_{1} \subseteq_{\mathrm{pre}} \tilde{\Omega}(k,i)]\mathbbm{1}[\omega_{2} \subseteq_{\mathrm{pre}} \tilde{\Omega}(l,j)]
\end{equation}
Note that since different random walks are chosen independently, all random variables $X_{k,l}$ are independent and therefore:
\begin{equation}
\mathrm{Var}(\mathbf{M}(i,j)) = \frac{1}{m^{2}}(\mathbb{E}[X_{1,1}^{2}]-(\mathbb{E}[X_{1,1}])^{2})
\end{equation}
Thus, from the unbiasedness of the estimator, we obtain:
\begin{equation}
\mathrm{Var}(\mathbf{M}(i,j)) = \frac{1}{m^{2}}(\mathbb{E}[X_{1,1}^{2}] - \mathbf{K}^{2}(i,j))
\end{equation}
It suffices to prove that: $\mathbb{E}[X_{1,1}^{2}] = \Lambda$. Note that we have:
\begin{align}
\begin{split}
\mathbb{E}[X_{1,1}^{2}] = \mathbb{E}[\sum_{x_{1} \in \mathrm{V}}\sum_{x_{2} \in \mathrm{V}}
\sum_{\omega_{1} \in \Omega(i,x_{1})}\sum_{\omega_{2} \in \Omega(j,x_{1})}\sum_{\omega_{3} \in \Omega(i,x_{2})}
\sum_{\omega_{4} \in \Omega(j,x_{2})} w(\omega_{1})w(\omega_{2})w(\omega_{3})w(\omega_{4}) \\
\mathbbm{1}[\omega_{1} \subseteq_{\mathrm{pre}} \tilde{\Omega}(1,i)]\mathbbm{1}[\omega_{2} \subseteq_{\mathrm{pre}} \tilde{\Omega}(1,j)] \\
\mathbbm{1}[\omega_{3} \subseteq_{\mathrm{pre}} \tilde{\Omega}(1,i)]\mathbbm{1}[\omega_{4} \subseteq_{\mathrm{pre}} \tilde{\Omega}(1,j)]
]
\end{split}
\end{align}

Thus we have:
\begin{align}
\begin{split}
\mathbb{E}[X_{1,1}^{2}] = \sum_{x_{1} \in \mathrm{V}}\sum_{x_{2} \in \mathrm{V}}
\sum_{\omega_{1} \in \Omega(i,x_{1})}\sum_{\omega_{2} \in \Omega(j,x_{1})}\sum_{\omega_{3} \in \Omega(i,x_{2})}
\sum_{\omega_{4} \in \Omega(j,x_{2})} w(\omega_{1})w(\omega_{2})w(\omega_{3})w(\omega_{4}) \\
\mathbb{P}[\mathbbm{1}[\omega_{1} \subseteq_{\mathrm{pre}} \tilde{\Omega}(1,i)]\mathbbm{1}[\omega_{2} \subseteq_{\mathrm{pre}} \tilde{\Omega}(1,j)] \\
\mathbbm{1}[\omega_{3} \subseteq_{\mathrm{pre}} \tilde{\Omega}(1,i)]\mathbbm{1}[\omega_{4} \subseteq_{\mathrm{pre}} \tilde{\Omega}(1,j)]]
\end{split}
\end{align}

We conclude  the proof, observing that:
\begin{equation}
\Gamma(\omega_{1},\omega_{2},\omega_{3},\omega_{4}) = \mathbb{P}\left[\mathbbm{1}[\omega_{1} \subseteq_{\mathrm{pre}} \tilde{\Omega}(1,i)]\mathbbm{1}[\omega_{2} \subseteq_{\mathrm{pre}} \tilde{\Omega}(1,j)] 
\mathbbm{1}[\omega_{3} \subseteq_{\mathrm{pre}} \tilde{\Omega}(1,i)]\mathbbm{1}[\omega_{4} \subseteq_{\mathrm{pre}} \tilde{\Omega}(1,j)]\right]
\end{equation}
\end{proof}
Completely analogous analysis can be conducted for $i=j$, but the formula is more convoluted since certain pairs of walks 
sampled from $i$ and $j$ are clearly no longer independent as being exactly the same (namely: the kth sampled walk from $i$ and the kth sampled walk from $j$ for $k=1,2,...,m$).

\subsection{Additional derivations leading to Equation \ref{eq:unbiased}}

For Reader's convenience, we will here explain in more detail the derivations leading to Eq. \ref{eq:unbiased}.
For the simplicity, we will assume that $m=1$ since for larger $m$ the formula is obtained simply by averaging over all pairs of random walks from node $i$ and $j$.

Note first that directly from the definition of GRFs and Algorithm 1, we get:

\begin{equation}
\phi(i)^{\top}\phi(j) = \sum_{x \in \mathrm{V}(G_{\mathbf{U}})}
\left (\sum_{r \in \mathcal{K}_{x}(\bar{\Omega}(1,i))}\mathrm{load}^{\bar{\Omega}(1,i)}(x,r) \right)
\left (\sum_{v \in \mathcal{K}_{x}(\bar{\Omega}(1,j))}\mathrm{load}^{\bar{\Omega}(1,j)}(x,v) \right),
\end{equation}
where $\mathcal{K}_{x}(\omega)$ for a given walk $\omega$ returns the set of these time indices when $\omega$ hits $x$ (first vertex of the walk gets time index zero, next one, time index one, etc.) and furthermore $\mathrm{load}^{\omega}(x,c)$ returns the increment of the load in $x$ added to the current $\mathrm{load}$ in time $c$ (see: Algorithm 1, while-loop, point 5). Now note that each $r$ can be identified with its corresponding prefix-subwalk $\omega_{1}$ and each $v$ can be identified with its corresponding prefix-subwalk $\omega_{2}$ 
Finally, observe that the increment of the load that $\omega_{1}$ contributes to is precisely: $w(\omega_{1})\prod_{s=1}^{l_{1}}\frac{(p^{i}_{1,s})^{-1}}{1-p_{\mathrm{term}}}$ and the increment of the load that $\omega_{2}$ contributes to is precisely: $w(\omega_{2})\prod_{t=1}^{l_{2}}\frac{(p^{j}_{1,t})^{-1}}{1-p_{\mathrm{term}}}$ (see: Algorithm 1, while-loop, point 4). That gives us Eq. \ref{eq:unbiased}.

\end{document}